\documentclass{elsarticle}
\journal{Stochastic Processes and their Applications}

\usepackage{mathrsfs} 
\usepackage{amssymb} 
\usepackage{amsmath}
\usepackage{amsthm} 
\usepackage{dsfont} 
\usepackage{mathtools} 
\usepackage{ascmac} 
\usepackage{color}
\usepackage{bm} 
\usepackage{comment}
\usepackage[all]{xy} 
\usepackage{url}

\newcommand{\R}{\mathbb{R}}
\newcommand{\N}{\mathbb{N}}

\newcommand{\F}{\mathcal{F}}
\newcommand{\D}{\mathcal{D}}

\numberwithin{equation}{section} 

\newtheorem {rem}{Remark}[section]

\newtheorem {thm}[rem]{Theorem}
\newtheorem {lem}[rem]{Lemma}

\newtheorem {prop}[rem]{Proposition}

\newtheorem {assu}[rem]{Assumption}

\usepackage[top=2cm, bottom=2cm, left=2cm, right=2cm]{geometry}

\begin{document}
	
\begin{frontmatter}
	\title{Uniform Generalization Bound on Time and Inverse Temperature for Gradient Descent Algorithm and its Application to Analysis of Simulated Annealing}
	\author{Keisuke Suzuki}
	\address{Biometrics Research Laboratories, NEC Corporation, 1753, Shimonumabe, Nakahara-Ku,	Kawasaki, Kanagawa 211-8666, Japan}
	\ead{keisuke.suzuki.334@nec.com}
	\begin{abstract}
		In this paper, we propose a novel uniform generalization bound on the time and inverse temperature for stochastic gradient Langevin dynamics (SGLD) in a non-convex setting.
		While previous works derive their generalization bounds by uniform stability, we use Rademacher complexity to make our generalization bound independent of the time and inverse temperature. 
		Using Rademacher complexity, we can reduce the problem to derive a generalization bound on the whole space to that on a bounded region and therefore can remove the effect of the time and inverse temperature from our generalization bound. 
		As an application of our generalization bound, an evaluation on the effectiveness of the simulated annealing in a non-convex setting is also described. 
		For the sample size $n$ and time $s$, we derive evaluations with orders $\sqrt{n^{-1} \log (n+1)}$ and $|(\log)^4(s)|^{-1}$, respectively. 
		Here, $(\log)^4$ denotes the $4$ times composition of the logarithmic function.
	\end{abstract}
\end{frontmatter}

\section{Introduction}

Numerical calculation methods have become practical due to the development of computers, and therefore it has become more and more important to guarantee their performance theoretically. 
In fact, almost all algorithms that achieve numerical solutions include hyperparameters, which are arbitrarily set by the user. 
In general, the setting of hyperparameters greatly affects the performance of the algorithm. 
In particular, it is important to derive an explicit evaluation of the relationship between hyperparameter settings and algorithm performance to determine the optimal hyperparameter settings. 

For stochastic gradient Langevin dynamics (SGLD), which is one of the typical optimization algorithms, \citep{Bertsekas, Kumar, Moulines, Abbasi, Mackey2, Ge, Ge2, Qian, Majka, Kavis, Mou, Taiji, Ragi, Cucumber, Taiji2, Xu, Zhang, Liang}, have derived evaluations on the effectiveness of SGLD from which we can choose appropriate hyperparameter settings. 
Let $\mathcal{Z}$ be the set of all data points and $\ell(w; z) : \R^d \times \mathcal{Z} \to [0, \infty)$ denote the loss on $z \in \mathcal{Z}$ for a parameter $w \in \R^d$. 
$z_1, \dots, z_n$ are independent and identically distributed (IID) samples generated from the distribution $\D$ on $\mathcal{Z}$, and we define the empirical loss by $L_n(w) = \frac{1}{n} \sum_{i=1}^n \ell(w; z_i)$. 
Then, for the step size $\eta > 0$ and the inverse temperature $\beta > 0$, SGLD is defined as follows. 
\begin{align}
	\label{NO_SGD}
	X_{k+1}^{(n, \eta)} 
	= X_k^{(n, \eta)} - \eta \nabla L_n(X_k^{(n, \eta)}) + \sqrt{2 \eta / \beta} \epsilon_k,\quad k \geq 0.
\end{align}
Here, $\epsilon_k$ are IIDs, each of which obeys the $d$-dimensional standard normal distribution. 
Assuming the dissipativity of the loss $\ell(w; z)$ and Lipschitz continuity of its gradient, \citep{Cucumber} showed the following evaluation to (\ref{NO_SGD}), where $L(w) = E_{z \sim \D}[\ell(w; z)]$ denotes the expected loss for the parameter $w$ and $C_i > 0$ are constants independent of $\eta$, $n$, $\beta$, and $k$. 
\begin{align}
	\label{Main_Previous_Cucumber}
	E[L(X_k^{(n, \eta)})] - \min_{w \in \R^d} L(w)
	\leq C_1 \left( \frac{e^{C_2 \beta}}{n} + \sqrt{\eta} e^{C_2 \beta} + \exp \left\{ C_2 \beta - \frac{C_3 k \eta}{e^{C_4 \beta}} \right\} + \frac{\log(\beta +1)}{\beta} \right).
\end{align}
According to (\ref{Main_Previous_Cucumber}), by determining hyperparameters in the order of $\beta$, $\eta$, $k$, and $n$, SGLD (\ref{NO_SGD}) can minimize the expected loss with arbitrary accuracy. 

However, SGLD (\ref{NO_SGD}) always contains a constant error with order $\beta^{-1} \log (\beta+1)$ since it uses the fixed inverse temperature $\beta$. 
The algorithm that increases the inverse temperature and decreases the step size with time evolution to remove this error term is called as simulated annealing (SA). 
For SA, as (\ref{Main_Previous_Cucumber}) indicates, by setting increase and decrease rates of the inverse temperature and step size properly, the error terms caused by the setting of hyperparameters except for $n$ vanish with time evolution. 
Whereas, as in the first term in the R.H.S of (\ref{Main_Previous_Cucumber}), previous works on SGLD \citep{Mou, Ragi, Cucumber} only have derived generalization bounds, which are bounds for $n$, that explode as $\beta$ tends to infinity. 
Hence, it seems that the sample size $n$ should increase with time evolution to control the effect of $\beta$ when SA is applied. 
However, in general, the sample size $n$ has its upper bound and we cannot take $n$ arbitrarily large enough to control the effect of $\beta$. 
In fact, previous works on SA \citep{Laarhoven, Abbas, Bouttier, Zhou, Mitter, Sheu, Lecchini, Locatelli} have not derived generalization bounds, which indicates it is difficult to derive practical generalization bounds to the SA algorithm. 

The first main result in this paper, Theorem \ref{Thm_Uniform_GenBound}, is a refined generalization bound to SGLD (\ref{NO_SGD}). 
While previous works \citep{Ragi, Cucumber} use uniform stability \citep{Hadt}, we use Rademacher complexity to make our generalization bound independent of the time and inverse temperature. 

The second main result, Theorem \ref{Thm_Bound_For_SA}, is the evaluation of the same form as (\ref{Main_Previous_Cucumber}) on the effectiveness of the SA algorithm. 
Using the generalization bound capable of increasing the inverse temperature, we can derive a practical generalization bound to the SA algorithm.

This paper is organized as follows. 
In Section \ref{SEC_Main_Result}, we give accurate statements of our main results of uniform generalization bound on the time and inverse temperature for SGLD algorithm and its application to the evaluation on the SA algorithm. 
Sections \ref{SEC_Proof_Main1} and \ref{SEC_Proof_Main2} are devoted to the proof of the first and second main results, respectively. 
Finally, the results used in Sections \ref{SEC_Proof_Main1} and \ref{SEC_Proof_Main2} are stated and proved in Appendix. 

\section{Main Result}
\label{SEC_Main_Result}

To formulate our first main result, we introduce the following notations. 
$S = (z_1, \dots, z_n) \in \mathcal{Z}^n$ are IIDs generated from the distribution $\D$ on $\mathcal{Z}$. 
Let $\ell(w; z)$ be a loss function and we define the expected loss and the empirical loss by $L(w) = E_{z \sim \D}[\ell(w; z)]$ and $L_n(w) = \frac{1}{n} \sum_{i=1}^n \ell(w; z_i)$, respectively. 
For a $d$-dimensional Brownian motion $W$, the initial value $x_0 \in \R^d$, and the inverse temperature $\beta > 0$, we consider
\begin{align}
	\label{NO_Conti_Emp_SGLD}
	dX_t^{(n)} = - \nabla L_n(X_t^{(n)}) dt + \sqrt{2 / \beta} dW_t,\quad X_0^{(n)} = x_0. 
\end{align}
We impose the following assumption on the loss function $\ell(w; z)$, and therefore the drift coefficient $- \nabla L_n$ of (\ref{NO_Conti_Emp_SGLD}). 

\begin{assu}
	\label{Assum_Main1}
	The loss $\ell(w; z)$ is nonnegative and satisfies $\sup_{z \in \mathcal{Z}} |\ell(0; z)| \leq B$ and $\sup_{z \in \mathcal{Z}} \|\nabla \ell(0; z)\|_{\R^d} \leq A$ for some $A, B > 0$.
	Thus, the expected loss $L(w) = E_{z \sim \D}[\ell(w; z)]$ is well-defined. 
	In addition, $\ell(\cdot; z) \in C^1(\R^d; \R)$ satisfies the following two conditions for all $z \in \mathcal{Z}$.
	\begin{itemize}
		\item[(1)] $(m, b)$-dissipative for some $m, b > 0$. Here, $H \in C^1(\R^d; \R)$ is said to be $(m, b)$-dissipative when the following inequality holds. 
		\begin{align}
			\label{Def_of_Dissipative}
			\langle \nabla H(x), x \rangle_{\R^d} 
			\geq m \| x \|_{\R^d}^2 - b,\qquad x \in \R^d.
		\end{align} 
		\item[(2)] $M$-smooth for some $M > 0$. Here, $H \in C^1(\R^d; \R)$ is said to be $M$-smooth when the following inequality holds. 
		\begin{align}
			\label{Def_of_M-smooth}
			\| \nabla H(x) - \nabla H(y) \|_{\R^d} 
			\leq M \| x - y \|_{\R^d},\qquad x, y \in \R^d.
		\end{align} 
	\end{itemize}
\end{assu}

Under the notation of (\ref{NO_SGD}) and (\ref{NO_Conti_Emp_SGLD}), previous works on SGLD \citep{Mou, Ragi, Cucumber} use the following decomposition. 
\begin{align}
	&E[L(X_k^{(n, \eta)})] - \min_{w \in \R^d} L(w) \notag \\
	&\quad= \left\{ E[L(X_k^{(n, \eta)})] - E[L(X_{k \eta}^{(n)})] \right\} + \left\{ E[L(X_{k \eta}^{(n)})] - E[L_n(X_{k \eta}^{(n)})] \right\} + \left\{ E[L_n(X_{k \eta}^{(n)})] - \min_{w \in \R^d} L(w) \right\}. \label{EQ_Decompose_SGD_Object}
\end{align}
Then, the generalization bound for SGLD corresponds to the bound to the second term in the R.H.S of (\ref{EQ_Decompose_SGD_Object}). 

The first main result in this paper is a uniform evaluation on the time and inverse temperature to the second term in the R.H.S of (\ref{EQ_Decompose_SGD_Object}). 
In the following, we denote $f = O_\alpha(g)$ if there exists a constant $C_\alpha > 0$ that depends only on $\alpha$ such that $f \leq C_\alpha g$ holds. 
Similarly, we denote $f = \Omega_\alpha(g)$ when $f \geq C_\alpha g$ holds.  

\begin{thm}
	\label{Thm_Uniform_GenBound}
	Under Assumption \ref{Assum_Main1}, for sufficiently large $\beta > 0$ and $\alpha_0 = (m, b, M, A, B, d)$, the following inequality holds. 
	\begin{align*}
		| E[L(X_t^{(n)})] - E[L_n(X_t^{(n)})] | 
		\leq O_{\alpha_0}\left( \sqrt{\frac{\log (n+1)}{n}} + (1 + \| x_0 \|_{\R^d}^3) \exp \left\{ - \frac{t}{e^{ \Omega_{\alpha_0}(\beta)}} + O_{\alpha_0}(\beta) \right\} + (1 + \| x_0 \|_{\R^d}^2) \sqrt{ \frac{\log (\beta + 1)}{\beta}} \right).
	\end{align*}
\end{thm}

Next, to formulate our second main result, we introduce the following notations. 
$\gamma : [0, \infty) \to (0, \infty)$ is a strictly increasing function and for a monotone decreasing sequence $\eta = \{ \eta_k \}_{k=1}^\infty$, we set $T_k \coloneqq \sum_{j=1}^k \eta_j$. 
Denoting $\phi^{(\eta)}(t) = \sum_{k=1}^{\infty} T_k \chi_{(T_k, T_{k+1}]}(t)$, we define the SA $Z^{(n)}$ and its discretization $Z^{(n, \eta)}$ with initial values $x_0 \in \R^d$ as follows. 
Here, $\chi_{\Gamma}$ denote the indicator function of $\Gamma$. 
\begin{align}
	dZ_t^{(n)} &= - \nabla L_n(Z_t^{(n)}) dt + \sqrt{2 / \gamma(s)} dW_s,\quad Z_0^{(n)} = x_0, \label{NO_Conti_Emp_SA} \\
	dZ_t^{(n, \eta)} &= - \nabla L_n(Z_{\phi^{(\eta)}(t)}^{(n, \eta)}) dt + \sqrt{2 / \gamma(s)} dW_s,\quad Z_0^{(n, \eta)} = x_0. \label{NO_Discrete_Emp_SA}
\end{align}
Finally, for each $s \geq 0$, we define the function $\alpha(s, \cdot)$ by 
\begin{align}
	\label{Def_alpha_s_t}
	\int_s^{\alpha(s, t)} \frac{\gamma(s)}{\gamma(u)} du
	= t.
\end{align}
The properties of $\alpha(s, \cdot)$ is described in Lemma \ref{Lem_beta_s_t}.
For the function $\gamma$ and the sequence $\eta$, we impose the following assumption. 

\begin{assu}
	\label{Assum_Main2}
	For sufficiently large $t > 0$, $\gamma(t) = (\log)^3 (t)$ holds. 
	Here, $(\log)^k$ denotes the $k$ times composition of the logarithmic function. 
	Furthermore, $\eta_k = 1 / k$ for all $k \in \N$.   
	Thus, $\lim_{k \to \infty} T_k = \infty$ holds. 
\end{assu}

The second main result in this paper is an evaluation on the effectiveness of the SA. 

\begin{thm}
	\label{Thm_Bound_For_SA}
	Under Assumptions \ref{Assum_Main1} and \ref{Assum_Main2}, for sufficiently large $s > 0$ and $\alpha_1 = (m, b, M, A, B, \gamma(0), d)$, the following inequalities hold. 
	\begin{align}
		E\left[L_n(Z_{\alpha(s, s^{2/3})}^{(n)})\right] - \min_{w \in \R^d} L(w) 
		&\leq O_{\alpha_1} \left(\frac{1+\| x_0 \|_{\R^d}^4}{(\log)^4(s)} \right), \label{Main_Result_TimeBound_SA} \\
		\left| E\left[L(Z_{\alpha(s, s^{2/3})}^{(n)})\right] - E\left[L_n(Z_{\alpha(s, s^{2/3})}^{(n)})\right] \right|
		&\leq O_{\alpha_1} \left( \sqrt{\frac{\log (n+1)}{n}} + \frac{1+\| x_0 \|_{\R^d}^4}{(\log)^4(s)} \right), \label{Main_Result_GenBound_SA} \\
		\left| E\left[L(Z_{\alpha(s, s^{2/3})}^{(n)})\right] - E\left[L(Z_{\alpha(s, s^{2/3})}^{(\eta, n)})\right] \right|
		&\leq O_{\alpha_1} \left( (1+\| x_0 \|_{\R^d}^3) \exp \{- \Omega_{\alpha_1}(s^{1/2})\} \right). \label{Main_Result_StepBound_SA}
	\end{align}
	In particular, we have 
	\begin{align*}
		E\left[L(Z_{\alpha(s, s^{2/3})}^{(\eta, n)})\right] - \min_{w \in \R^d} L(w)
		\leq O_{\alpha_1} \left( \sqrt{\frac{\log (n+1)}{n}} + \frac{1+\| x_0 \|_{\R^d}^4}{(\log)^4(s)} \right).
	\end{align*}
\end{thm}

Almost all of the previous works on SA \citep{Laarhoven, Abbas, Bouttier, Zhou, Mitter, Sheu, Lecchini, Locatelli} consider the optimization problem on the discrete or bounded space. 
In addition, these works only show that the SA algorithm approaches any fixed neighborhood of minimizers of objective functions. 
Therefore, Theorem \ref{Thm_Bound_For_SA} is novel in that it considers the optimization problem on $\R^d$ and derives the evaluation equivalent to (\ref{Main_Previous_Cucumber}). 

\begin{rem}
	The sequence $\{ Z_{T_k}^{(n, \eta)} \}_{k=0}^\infty$, which is constructed by extracting values from $Z^{(n, \eta)}$ at each $T_k$, has the same law as the sequence
	\begin{align}
		\label{Def_True_Discrete_SA}
		\tilde{Z}_{k+1}^{(n, \eta)}
		= \tilde{Z}_k^{(n, \eta)} - \eta_k \nabla L_n(\tilde{Z}_k^{(n, \eta)}) + \sqrt{\tilde{\eta_k}} \epsilon_k,\quad \tilde{Z}_0^{(n, \eta)} = x_0.
	\end{align}
	Here, $\tilde{\eta}_k = \int_{T_k}^{T_{k+1}} 2 \gamma(t)^{-1} dt$. 
\end{rem}

\section{Proof of Theorem \ref{Thm_Uniform_GenBound}}
\label{SEC_Proof_Main1}

In this section, we prove our first main result, Theorem \ref{Thm_Uniform_GenBound}. 
Generalization bounds given in \citep{Ragi, Cucumber} are based on uniform stability \citep{Hadt}. 
In the following, by Rademacher complexity, we derive a generalization bound capable of increasing the inverse temperature. 
While our bound is capable of increasing the inverse temperature, the order of it degrades from $n^{-1}$ to $\sqrt{n^{-1}\log (n+1)}$ compared with the results by previous works.

First, we show Lemma \ref{Lem_GenBound_Fixed_R} below, which is based on existing results Theorems \ref{Thm_Covering_Num_Bound} and \ref{Thm_Rademach_GenBound} on Rademacher complexity. 

\begin{lem}
	\label{Lem_GenBound_Fixed_R}
	For any $R > 0$, we have 
	\begin{align*}
		E\left[ \sup_{\| w \|_{\R^d} \leq R} | L(w) - L_n(w) | \right] 
		\leq O_{M, A, B, d, R} \left( \sqrt{\frac{\log (n+1)}{n}} \right).
	\end{align*}
\end{lem}

\begin{proof}
	Let $\F = \{ \ell(w; \cdot) \mid \| w \|_{\R^d} \leq R \}$.
	If $\| w \|_{\R^d} \leq R$, then $\sup_{z \in \mathcal{Z}} \| \nabla \ell(w; z) \|_{\R^d} \leq A + M R$ holds. 
	Thus, for any $\| w \|_{\R^d}, \| v \|_{\R^d} \leq R$, we have $| \ell(w; z) - \ell(v; z) | \leq (A + M R) \| w -v \|_{\R^d}$. 
	Furthermore, for any $\delta > 0$, $\{ w \in \R^d \mid \| w \|_{\R^d} \leq R \}$ can be covered by $(\delta^{-1} R \sqrt{d} + 1)^d$ closed balls with radius $\delta$. 
	Therefore, with the notation in Theorem \ref{Thm_Covering_Num_Bound}, we obtain
	\begin{align*}
		C(\F, n^{-1}, \| \cdot \|_{1, S}) \leq (n R (A + M R) \sqrt{d} + 1)^d.
	\end{align*}
	Similarly, if $\| w \|_{\R^d} \leq R$, then $|\ell(w; z)| \leq B + (A + M R) R$ holds, and therefore
	\begin{align*}
		\sup_{f \in \F} \left( \frac{1}{n} \sum_{i=1}^n f(z_i)^2 \right)^{1/2}
		\leq B + (A + M R) R.
	\end{align*}
	As a result, applying Theorem \ref{Thm_Covering_Num_Bound} to $\varepsilon = n^{-1}$, we obtain
	\begin{align*}
		\hat{R}_n(\F, S) 
		\leq \frac{1}{n} + \{ B + (A + M R) R \} \sqrt{\frac{2 d \log (n R (A + M R) \sqrt{d} + 1)}{n}}.
	\end{align*}
	Theorem \ref{Thm_Rademach_GenBound} completes the proof. 
\end{proof}

Lemma \ref{Lem_GenBound_Fixed_R} proves Theorem \ref{Thm_Uniform_GenBound} as follows. 
Let $R = \sqrt{2 m^{-1} (2 + b \log 2)}$. 
If $\| x \|_{\R^d} > R$, then we have $L_n(x) - \min_{w \in \R^d} L_n(w) \geq 1$ by Lemma \ref{Lem_Lower_Bound_Dissipative_Func}. 
Thus, for the Gibbs measure $\pi_{\beta, L_n}(dw) \propto e^{- \beta L_n(w)} dw$, 
\begin{align*}
	P(\| X_t^{(n)} \|_{\R^d} > R) 
	\leq E[L_n(X_t^{(n)}) - \pi_{\beta, L_n}(L_n)] + E\left[ \pi_{\beta, L_n}(L_n) - \min_{w \in \R^d} L_n(w) \right]
\end{align*}
holds. 
Therefore, Lemma A.15, Proposition 4.1 in \citep{Cucumber} and Proposition 11 in \citep{Ragi} yield
\begin{align}
	\label{IQ_LD_Away_Bound}
	P(\| X_t^{(n)} \|_{\R^d} > R) 
	\leq O_{\alpha_0}\left( (1 + \| x_0 \|_{\R^d}^2) \exp \left\{ - \frac{t}{e^{\Omega_{\alpha_0}(\beta)}} + O_{\alpha_0}(\beta) \right\} + \frac{\log (\beta + 1)}{\beta} \right).
\end{align}
Since Lemmas \ref{Lem_Lower_Bound_Dissipative_Func} and \ref{Lp_bound_SGLD} yield $E[|L(X_t^{(n)}) - L_n(X_t^{(n)})|^2]^{1/2} \leq O_{\alpha_0}(1+\| x_0 \|_{\R^d}^2)$, 
\begin{align*}
	E[L(X_t^{(n)})] - E[L_n(X_t^{(n)})] 
	&\leq E\left[ \sup_{\| w \|_{\R^d} \leq R} | L(w) - L_n(w) | \right] + O_{\alpha_0} \left( (1+\| x_0 \|_{\R^d}^2) P(\| X_t^{(n)} \|_{\R^d} > R)^{1/2} \right)
\end{align*}
holds. 
The desired result follows from Lemma \ref{Lem_GenBound_Fixed_R} and (\ref{IQ_LD_Away_Bound}). 
\qed

\section{Proof of Theorem \ref{Thm_Bound_For_SA}}
\label{SEC_Proof_Main2}

In this section, we prove Theorem \ref{Thm_Bound_For_SA}. 
As for (\ref{Main_Result_TimeBound_SA}) and (\ref{Main_Result_GenBound_SA}), we follow the same scheme as in \citep{Sheu}. 
That is, for a sufficiently large time $s$, we approximate the SA (\ref{NO_Conti_Emp_SA}) by SGLD
\begin{align}
	\label{NO_Conti_Approx_SA}
	dY^{(n)}(s, t) 
	= -\nabla L_n(Y^{(n)}(s, t)) dt + \sqrt{2 / \gamma(s)} dW_t
\end{align}
and reduce the problem of deriving bounds for (\ref{NO_Conti_Emp_SA}) to that for (\ref{NO_Conti_Approx_SA}). 
As an approximation of (\ref{NO_Conti_Emp_SA}) by (\ref{NO_Conti_Approx_SA}), we use Lemma \ref{Lem_Diff_Between_Y_Z}, which is a refinement of Lemma 2 in \citep{Sheu}. 
On the other hand, we prove (\ref{Main_Result_StepBound_SA}) directly using approximate reflection coupling (ARC) proposed in \citep{Cucumber}. 

\subsection{Proof of (\ref{Main_Result_TimeBound_SA})}

For sufficiently large $s > 0$, we consider the decomposition
\begin{align}
	E[L_n(Z_{\alpha(s, s^{2/3})})] - \min_{w \in \R^d} L(w) 
	&= E\left[L_n(Z_{\alpha(s, s^{2/3})}) - \min_{w \in \R^d} L(w); \{ L_n(Z_s^{(n)}) \leq (\log)^4 (s) \}\right] \notag \\
	&\quad+ E\left[L_n(Z_{\alpha(s, s^{2/3})}) - \min_{w \in \R^d} L(w); \{ L_n(Z_s^{(n)}) > (\log)^4 (s) \}\right]. \label{EQ_Decompose_TimeBound_SA}
\end{align}
By Chebyshev's inequality and Lemmas \ref{Lem_Lower_Bound_Dissipative_Func} and \ref{Lp_bound_SGLD}, the second term in the R.H.S. of (\ref{EQ_Decompose_TimeBound_SA}) has the desired bound. 
To derive the bound for the first term, fix arbitrarily $x \in \R^d$ so that $L_n(x) \leq (\log)^4 (s)$ holds. 
In addition, suppose that $Y^{(n)}(s, \cdot)$ defined by (\ref{NO_Conti_Approx_SA}) has an initial value $x$. 
Then, Lemma \ref{Lem_Diff_Between_Y_Z} yields 
\begin{align*}
	\left| E[L_n(Z_{\alpha(s, s^{2/3})}) \,|\, S, Z_s^{(n)} = x] - E\left[ L_n(Y^{(n)}(s, s^{2/3})) \,\big|\, S \right] \right|
	\leq O_{\alpha_1}\left( \frac{1 + \| x \|_{\R^d}^2}{\sqrt{(\log)^2 (s)}} \right).
\end{align*}
Furthermore, applying Lemma A.15, Proposition 4.1 in \citep{Cucumber} and Proposition 11 in \citep{Ragi} to the R.H.S of 
\begin{align*}
	E\left[ L_n(Y^{(n)}(s, s^{2/3})) \,\big|\, S \right] - \min_{w \in \R^d} L_n(w)
	&\leq \left\{ E\left[ L_n(Y^{(n)}(s, s^{2/3})) \,\big|\, S \right] - \pi_{\gamma(s), L_n}(L_n) \right\} + \left\{ \pi_{\gamma(s), L_n}(L_n) - \min_{w \in \R^d} L_n(w) \right\}, 
\end{align*}
we obtain
\begin{align*}
	E\left[ L_n(Y^{(n)}(s, s^{2/3})) \,\big|\, S \right] - \min_{w \in \R^d} L(w)
	\leq O_{\alpha_1} \left( (1+\| x \|_{\R^d}^2) \exp \left\{ - \frac{s^{2/3}}{e^{\Omega_{\alpha_1}(\gamma(s))}} + O_{\alpha_1}(\gamma(s)) \right\} + \frac{\log(\gamma(s) + 1)}{\gamma(s)} \right).
\end{align*}
In particular, since $\gamma(s) = (\log)^3 (s)$ for sufficiently large $s > 0$, 
\begin{align*}
	E[L_n(Z_{\alpha(s, s^{2/3})}) \,|\, S, Z_s^{(n)} = x] - \min_{w \in \R^d} L_n(w)
	\leq O_{\alpha_1}\left( \frac{1 + \| x \|_{\R^d}^2}{(\log)^4(s)} \right)
\end{align*}
holds. 
In addition, $E[\min_{w \in \R^d} L_n(w)] \leq \min_{w \in \R^d} L(w)$ holds. 
Therefore, integrating both sides with respect to $P(Z_s^{(n)} \in dx)$ on $\{ x \in \R^d \mid L_n(x) \leq (\log)^4(s) \}$ and taking expectation on $S$, Lemma \ref{Lp_bound_SGLD} proves the desired result. 
\qed

\subsection{Proof of (\ref{Main_Result_GenBound_SA})}

As in the proof of (\ref{Main_Result_TimeBound_SA}), for sufficiently large $s > 0$, we consider the decomposition
\begin{align}
	E\left[L(Z_{\alpha(s, s^{2/3})}^{(n)})\right] - E\left[L_n(Z_{\alpha(s, s^{2/3})}^{(n)})\right]
	&= E\left[L(Z_{\alpha(s, s^{2/3})}^{(n)}) - L_n(Z_{\alpha(s, s^{2/3})}^{(n)}); \{ L_n(Z_s^{(n)}) \leq (\log)^4 (s) \}\right] \notag \\
	&\quad+ E\left[L(Z_{\alpha(s, s^{2/3})}^{(n)}) - L_n(Z_{\alpha(s, s^{2/3})}^{(n)}); \{ L_n(Z_s^{(n)}) > (\log)^4 (s) \}\right]. \label{EQ_Decompose_GenBound_SA}
\end{align}
Then, the second term in the R.H.S of (\ref{EQ_Decompose_GenBound_SA}) has the desired bound. 
To derive the bound for the first term, fix arbitrarily $x \in \R^d$ so that $L_n(x) \leq (\log)^4 (s)$ and suppose that $Y^{(n)}(s, \cdot)$ defined by (\ref{NO_Conti_Approx_SA}) has an initial value $x$.   
Then, we have by Theorem \ref{Thm_Uniform_GenBound}
\begin{align*}
	\left| E\left[ L(Y^{(n)}(s, s^{2/3})) \,\big|\, S \right] - E\left[ L_n(Y^{(n)}(s, s^{2/3})) \,\big|\, S \right] \right|
	\leq O_{\alpha_1}\left( \sqrt{\frac{\log (n+1)}{n}} + (1 + \| x \|_{\R^d}^3) \sqrt{\frac{\log(\gamma(s) + 1)}{\gamma(s)}} \right).
\end{align*}
Therefore, we can prove (\ref{Main_Result_GenBound_SA}) in a similar manner to the proof of (\ref{Main_Result_TimeBound_SA}). 

\subsection{Proof of (\ref{Main_Result_StepBound_SA})}

(\ref{Main_Result_GenBound_SA}) can be proved in the same way as Theorem 2.4 (2) in \citep{Cucumber}. 
To explain this, we introduce the following notations. 
For $p > 0$, define $V_p : \R^d \to \R$ by $V_p(x) = \| x \|_{\R^d}^p$ and let $\bar{V}_p(x) = 1 + V_p(x)$. 
For constants $C(p)$ and $\lambda(p)$ defined by
\begin{align}
	\label{p-th_Lyap_Constant}
	\lambda(p) 
	= \frac{mp}{2},\quad 
	C(p) 
	= \lambda(p) \left\{ \frac{2}{m} \left( \frac{d + p-2}{\gamma(0)} + b \right) \right\}^{p/2},
\end{align} 
let  
\begin{align*}
	C = C(2) + \lambda(2),\qquad 
	\lambda = \lambda(2).
\end{align*}
For sets
\begin{align*}
	S_1 
	&\coloneqq \{ (x, y) \in \R^d \times \R^d \mid \bar{V}_2(x) + \bar{V}_2(y) \leq 2 \lambda^{-1} C \}, \\ 
	S_2
	&\coloneqq \{ (x, y) \in \R^d \times \R^d \mid \bar{V}_2(x) + \bar{V}_2(y) \leq 4C(1+\lambda^{-1}) \}, 
\end{align*}
let $R_1$ and $R_2$ be the diameters of $S_1$ and $S_2$, respectively, where the diameter of a set $\Gamma \subset \R^d$ is defined by $\sup_{x, y \in \Gamma} \| x - y \|_{\R^d}$. 

Fix $t > 0$ and we define $\kappa_t$ and $Q(\kappa_t)$ by
\begin{align}
	\label{Condi_On_kappa}
	\kappa_t 
	\coloneqq \min \left\{ \frac{1}{2}, \frac{2}{C \gamma(t) (e^{2 R_1} - 1 - 2 R_1)} \exp \left\{ - \frac{M \gamma(t)}{8} R_1^2 \right\} \right\} 
	\in (0, 1)
\end{align}
and
\begin{align*}
	Q(\kappa)
	\coloneqq \sup_{x \in \R^d} \frac{\| \nabla \bar{V}_2(x) \|_{\R^d}}{\max \{ \bar{V}_2(x), \kappa_t^{-1} \}} 
	= \sup_{x \in \R^d} \frac{2 \| x \|_{\R^d}}{\max \{ 1 + \| x \|_{\R^d}^2, \kappa_t^{-1} \}} 
	= 2 \sqrt{ \kappa_t - \kappa_t^2 } \in (0, 1],
\end{align*}
respectively. 
In addition, we define functions $\varphi_t, \Phi_t : [0, \infty) \to [0, \infty)$ by
\begin{align*}
	\varphi_t(r) 
	\coloneqq \exp \left( - \frac{M \gamma(t)}{8} r^2 - 2 Q(\kappa_t) r \right),\quad
	\Phi_t(r) = \int_0^r \varphi_t(s) ds. 
\end{align*}
For constants $\zeta_t$, $\xi_t$ and $c_t$ defined by
\begin{align}
	\label{DefOfBetaAndXi}
	\frac{1}{\zeta_t} 
	\coloneqq \int_0^{R_2} \Phi_t(s) \varphi_t(s)^{-1} ds,\quad
	\frac{1}{\xi_t} 
	\coloneqq \int_0^{R_1} \Phi_t(s) \varphi_t(s)^{-1} ds,\quad
	c_t 
	\coloneqq \min \left\{ \frac{\zeta_t}{\gamma(t)}, \frac{\lambda}{2}, 2 C \lambda \kappa_t \right\}, 
\end{align}
let 
\begin{align*}
	g_t(r) 
	\coloneqq 1 - \frac{\zeta_t}{4} \int_0^{\min \{r, R_2\}} \Phi_t(s) \varphi_t(s)^{-1} ds - \frac{\xi_t}{4} \int_0^{\min \{r, R_1\}} \Phi_t(s) \varphi_t(s)^{-1} ds. 
\end{align*}
Furthermore, for
\begin{align*}
	f_t(r) 
	\coloneqq 
	\begin{cases}
		{\displaystyle\int_0^{\min \{r, R_2\}} \varphi_t(s) g_t(s) ds}, &r \geq 0 \\
		r, &r<0
	\end{cases}
\end{align*}
and $U_t(x, y) \coloneqq 1 + \kappa_t \bar{V}_2(x) + \kappa_t \bar{V}_2(y)$, let
\begin{align*}
	\rho_{2, t}(x, y) 
	= f_t(\| x-y \|_{\R^d}) U_t(x, y),\qquad x, y \in \R^d. 
\end{align*}
Finally, for probability measures $\mu$ and $\nu$ on $\R^d$, denoting the set of all coupling between them by $\Pi(\mu, \nu)$, let
\begin{align}
	\label{W_rho2}
	\mathcal{W}_{\rho_{2, t}}(\mu, \nu) 
	\coloneqq \inf_{\gamma \in \Pi(\mu, \nu)} \int_{\R^d \times \R^d} \rho_{2, t}(x, y) \gamma(dx dy). 
\end{align}
Here, for random variables $Z_1$ and $Z_2$, $\mathcal{W}_{\rho_2}(\mathcal{L}(X_1), \mathcal{L}(X_2))$  may be abbreviated as $\mathcal{W}_{\rho_2}(X_1, X_2)$. 

With aforementioned notations, noting the monotonicity of the function $\gamma$, we can prove the following in the same way as Proposition 5.1 in \citep{Cucumber}. 

\begin{prop}
	\label{Prop_Discrete_Approx}
	For any $0 \leq s \leq t$, the following inequality holds. 
	\begin{align}
		\label{Main_Result_StepSizeBound}
		\mathcal{W}_{\rho_{2, t}}(Z_s^{(n)}, Z_s^{(\eta, n)})
		&\leq O_{\alpha_1} \left( (1 + \| x_0 \|_{\R^d}^3) e^{- c_t s} \int_0^s e^{ c_t u} \sum_{k=0}^\infty \sqrt{\eta_{k+1}} \chi_{(T_k, T_{k+1}]}(u) du \right). 
	\end{align}
\end{prop}

Proposition \ref{Prop_Discrete_Approx} proves (\ref{Main_Result_StepBound_SA}) as follows. 
Combining Lemma A.15 in \citep{Cucumber} and Proposition \ref{Prop_Discrete_Approx}, we obtain 
\begin{align}
	\left| E[L(Z_t^{(n)})] - E[L(Z_t^{(\eta, n)})] \right|
	\leq O_{\alpha_1} \left( (1 + \| x_0 \|_{\R^d}^3) e^{O_{\alpha_1}(\gamma(t))} e^{- c_t t} \int_0^t e^{ c_t u} \sum_{k=0}^\infty \sqrt{\eta_{k+1}} \chi_{(T_k, T_{k+1}]}(u) du \right). 
\end{align}
Let $k$ be the natural number such that $T_k < \sqrt{t} \leq T_{k+1}$. 
Then $e^{\sqrt{t}-1} \leq k + 1$ since $T_k = \sum_{j=1}^k j^{-1} \leq 1 + \log k$. 
Therefore, 
\begin{align*}
	\int_0^t e^{ c_t u} \sum_{k=0}^\infty \sqrt{\eta_{k+1}} \chi_{(T_k, T_{k+1}]}(u) du 
	&= \int_0^{\sqrt{t}} e^{ c_t u} \sum_{k=0}^\infty \sqrt{\eta_{k+1}} \chi_{(T_k, T_{k+1}]}(u) du + \int_{\sqrt{t}}^t e^{ c_t u} \sum_{k=0}^\infty \sqrt{\eta_{k+1}} \chi_{(T_k, T_{k+1}]}(u) du \\
	&\leq c_t^{-1} (e^{ c_t \sqrt{t}} - 1) + c_t^{-1} e^{- \frac{\sqrt{t}}{2} + \frac{1}{2}} (e^{ c_t t} - e^{ c_t \sqrt{t}})
\end{align*}
holds. 
In addition, by (\ref{DefOfBetaAndXi}) and Assumption \ref{Assum_Main2}, 
\begin{align*}
	c_t 
	= \Omega_{\alpha_1} \left( \exp \left\{ - |(\log)^2(t)|^{O_{\alpha_1}(1)} \right\} \right)
\end{align*}
holds for sufficiently large $t > 0$. 
Hence, 
\begin{align*}
	&\left| E[L(Z_t^{(n)})] - E[L(Z_t^{(\eta, n)})] \right| \\
	&\quad\leq (1 + \| x_0 \|_{\R^d}^3) O_{\alpha_1} \left( \exp \left\{ - \frac{\Omega_{\alpha_1}(t - \sqrt{t})}{|(\log)^2(t)|^{O_{\alpha_1}(1)}} + |(\log)^2(t)|^{O_{\alpha_1}(1)} \right\} + \exp \left\{ \frac{1}{2} + |(\log)^2(t)|^{O_{\alpha_1}(1)} - \frac{\sqrt{t}}{2} \right\}  \right). 
\end{align*}
Taking $t = \alpha(s, s^{2/3})$, (\ref{Main_Result_StepBound_SA}) follows from Lemma \ref{Lem_beta_s_t}. 
\qed

\appendix
\renewcommand{\thesection}{\Alph{section}}
\section{Appendix}

\subsection{Difference between SGLD and SA}


In this subsection, we prepare a result on the approximation of SA by SGLD (Lemma \ref{Lem_Diff_Between_Y_Z}), which is a refinement of Lemma 2 in \citep{Sheu}. 

Let $F \in C^1(\R^d; \R)$ be $(m, b)$-dissipative and $M$-smooth, and fix $s > 0$. 
We define the SA and SGLD along with the gradient $\nabla F$ of $F$ by 
\begin{align}
	dZ_t 
	&= - \nabla F(Z_t) dt + \sqrt{2 / \gamma(t)} dW_t, \label{EQ_SA_Gen} \\
	dY(s, t) 
	&= - \nabla F(Y(s, t)) dt + \sqrt{2 / \gamma(s)} dW_t, \label{EQ_Noisy_LD_Gen}
\end{align}
respectively. 
Furthermore, for given $0 < r_0 < r_1 < r_2$, let 
\begin{align}
	\Omega_{i, F} = \{ x \in \R^d \mid F(x) \leq r_i \},\quad 
	\partial \Omega_{i, F} = \{ x \in \R^d \mid F(x) = r_i \} \label{Def_Omega_Fixed}
\end{align}
and for two sets $\Gamma_1, \Gamma_2 \subset \R^d$, we denote the distance between them by
\begin{align}
	\label{Def_Distance_Between_Sets}
	{\rm dist}(\Gamma_1, \Gamma_2) 
	= \inf \{ \| x - y \|_{\R^d} \mid x \in \Gamma_1, y \in \Gamma_2 \}.
\end{align}
Finally, for any $s \geq 0$, the solutions of (\ref{EQ_Noisy_LD_Gen}) and 
\begin{align}
	dX_t 
	= - \nabla F(X_t) dt, \label{EQ_Non_Noisy_LD_Gen}
\end{align}
with the same initial values $x \in \R^d$ are denoted by $Y^x(s, \cdot)$ and $X^x$, and the path of $X^x$ until $t$ is denoted by $\Gamma_F^x(t) = \{ X_s^x \mid 0 \leq s \leq t \}$. 

\begin{lem}
	\label{Lem_Lower_Bound_Dissipative_Func}
	(Lemma 2 in \citep{Ragi})
	For any $c \in (0, 1)$ and $x \in \R^d$, 
	\begin{align*}
		F(c x) + \frac{1}{2} (1 - c^2) m \| x \|_{\R^d}^2 + b \log c
		\leq F(x) 
		\leq F(0) + \frac{1}{2} \| \nabla F(0) \|_{\R^d}^2 + \frac{M+1}{2} \| x \|_{\R^d}^2 
	\end{align*}
	holds. 
	In particular, for $r > 0$, $F(x) \geq r$ and $F(x) \leq r$ indicate 
	\begin{align}
		\| x \|_{\R^d}^2 &\geq \frac{2}{M+1} \left( r - F(0) - \frac{1}{2} \| \nabla F(0) \|_{\R^d}^2 \right) \label{IQ_Out_Omega}
	\end{align}
	and
	\begin{align}
		\| x \|_{\R^d}^2 &\leq \frac{4}{m} \left( r + \frac{1}{2} b \log 2 - \inf_{w \in \R^d} F(w) \right),  \label{IQ_In_Omega}
	\end{align}
	respectively. 
\end{lem}

\begin{proof}
	By Taylor's theorem
	\begin{align*}
		F(x) - F(c x) 
		&= \int_c^1 \langle x, \nabla F(t x) \rangle_{\R^d} dt
		= \int_c^1 \frac{1}{t} \langle t x, \nabla F(t x) \rangle_{\R^d} dt
		\geq \int_c^1 \frac{1}{t} \{ m t^2 \| x \|_{\R^d}^2 -b \} dt
		= \frac{1}{2} (1 - c^2) m \| x \|_{\R^d}^2 + b \log c, \\
		F(x) - F(0) 
		&= \int_0^1 \langle x, \nabla F(t x) \rangle_{\R^d} dt 
		\leq \| x \|_{\R^d} \int_0^1 (\| \nabla F(0) \|_{\R^d} + M \| x \|_{\R^d} t ) dt
		= \| \nabla F(0) \|_{\R^d} \| x \|_{\R^d} + \frac{M}{2} \| x \|_{\R^d}^2
	\end{align*}
	hold. 
	Taking $c = 1 / \sqrt{2}$, the rest of the statement follows. 
\end{proof}

\begin{lem}
	\label{Lem_Select_r0}
	For any $\delta > 0$, let
	\begin{align}
		\tilde{r}_0(\delta) &= \frac{M+1}{2} \delta + F(0) + \frac{1}{2} \| \nabla F(0) \|_{\R^d}^2. \label{Condi_on_r0}
	\end{align}
	Then, if $x \in \R^d$ satisfies $F(x) \geq \tilde{r}_0(2b / m)$, 
	\begin{align}
		\label{IQ_Select_r0}
		\| \nabla F(x) \|_{\R^d}^2 - \frac{2}{\beta} \Delta F(x) 
		\geq 0
	\end{align}
	holds for any $\beta \geq 4 M d / m b$. 
\end{lem}

\begin{proof}
	According to (\ref{IQ_Out_Omega}), we have 
	\begin{align*}
		\| x \|_{\R^d}^2 
		\geq \frac{2}{M+1} \left( \tilde{r}_0(2b / m) - F(0) - \frac{1}{2} \| \nabla F(0) \|_{\R^d}^2 \right)
		= \frac{2 b}{m}.
	\end{align*}
	Thus, by the $(m, b)$-dissipativity of $F$, 
	\begin{align*}
		\| \nabla F(x) \|_{\R^d} 
		\geq \frac{1}{\| x \|_{\R^d}} (m \| x \|_{\R^d}^2 - b) 
		=m \| x \|_{\R^d} - \frac{b}{\| x \|_{\R^d}}
		\geq \sqrt{\frac{m b}{2}}
	\end{align*}
	holds. 
	On the other hand, $M$-smoothness of $F$ indicates $\Delta F(x) \leq M d$. 
	Therefore, if $\beta \geq 4 M d / m b$, then (\ref{IQ_Select_r0}) holds. 
\end{proof}

\begin{lem}
	\label{Lem_Select_r1}
	For any $\delta > 0$, let
	\begin{align}
		\label{Condi_on_r1}
		\tilde{r}_1(r_0, \delta)
		\geq F(0) + \frac{1}{2} \| \nabla F(0) \|_{\R^d}^2 + \frac{4(M+1)}{m} \left( r_0 + \frac{m \delta^2}{4} + \frac{1}{2} b \log 2 - \inf_{w \in \R^d} F(w) \right). 
	\end{align}
	Then, if $r_1 \geq \tilde{r}_1(r_0, \delta)$, we have ${\rm dist}(\Omega_{0, F}, \partial \Omega_{1, F}) \geq \delta$. 
\end{lem}

\begin{proof}
	Fix arbitrarily $x \in \Omega_{0, F}$ and $v \in \R^d$ such that $\| v \|_{\R^d} < \delta$. 
	Then, Lemma \ref{Lem_Lower_Bound_Dissipative_Func} yields
	\begin{align*}
		F(x + v) 
		&\leq F(0) + \frac{1}{2} \| \nabla F(0) \|_{\R^d}^2 + \frac{M+1}{2} \| x + v \|_{\R^d}^2 \\
		&< F(0) + \frac{1}{2} \| \nabla F(0) \|_{\R^d}^2 + (M+1) (\| x \|_{\R^d}^2 + \delta^2)  \\ 
		&\leq F(0) + \frac{1}{2} \| \nabla F(0) \|_{\R^d}^2 + \frac{4(M+1)}{m} \left( r_0 + \frac{m \delta^2}{4} + \frac{1}{2} b \log 2 - \inf_{w \in \R^d} F(w) \right), 
	\end{align*}
	and therefore $x + v \in \Omega_{1, F}$ cannot hold. 
\end{proof}

\begin{lem}
	\label{Lem_Lower_Bound_delta0}
	With the notation of (\ref{Condi_on_r1}), let $r_1 \geq \tilde{r}_1(r_0, 1)$. 
	If we define   
	\begin{align}
		\label{Condi_on_Epsilon}
		\varepsilon = \frac{1}{2\sqrt{2}} \left\{ \| \nabla F(0) \|_{\R^d}^2 + \frac{4 M^2}{m} \left( r_1 + \frac{1}{2} b \log 2 - \inf_{w \in \R^d} F(w) \right) \right\}^{-1/2}
	\end{align}
	and 
	\begin{align}
		\label{Def_of_Delta0}
		\delta_0 
		= {\rm dist}\left(\Omega_{0, F}, {\textstyle \bigcup_{x \in \partial \Omega_{1, F}}} \Gamma_F^x(\varepsilon) \right), 
	\end{align}
	then, $\delta_0 \geq 1/2$ holds. 
\end{lem}

\begin{proof}
	For $x \in \partial \Omega_{1, F}$, 
	\begin{align}
		\label{EQ_Non_Nosy_Ito_Formula}
		F(X_t^x)
		= r_1 - \int_0^t \| \nabla F(X_s^x) \|_{\R^d}^2 ds,\quad t \geq 0
	\end{align}
	holds. 
	In particular, by $F(X_t^x) \leq r_1$ and (\ref{IQ_In_Omega}), we have
	\begin{align}
		\label{IQ_Upper_Bound_Non_Noisy}
		\| \nabla F(X_t^x) \|_{\R^d}^2
		\leq 2 \| \nabla F(0) \|_{\R^d}^2 + 2 M^2 \| X_t^x \|_{\R^d}^2
		\leq 2 \| \nabla F(0) \|_{\R^d}^2 + \frac{8 M^2}{m} \left( r_1 + \frac{1}{2} b \log 2 - \inf_{w \in \R^d} F(w) \right). 
	\end{align}
	Thus, for any $0 \leq s \leq \varepsilon$, 
	\begin{align*}
		\| X_s^x - x \|_{\R^d} 
		\leq \varepsilon \sup_{u \geq 0} \| \nabla F(X_u^x) \|_{\R^d}
		\leq \frac{1}{2}
	\end{align*}
	holds. 
	In particular, $x \in \partial \Omega_{1, F}$ indicates ${\rm dist}\left(\partial \Omega_{1, F}, {\textstyle \bigcup_{x \in \partial \Omega_{1, F}}} \Gamma_F^x(\varepsilon) \right) \leq 1 / 2$. 
	Therefore, 
	\begin{align*}
		{\rm dist}\left(\Omega_{0, F}, {\textstyle \bigcup_{x \in \partial \Omega_{1, F}}} \Gamma_F^x(\varepsilon) \right) 
		\geq {\rm dist}(\Omega_{0, F}, \partial \Omega_{1, F}) - {\rm dist}\left(\partial \Omega_{1, F}, {\textstyle \bigcup_{x \in \partial \Omega_{1, F}}} \Gamma_F^x(\varepsilon) \right) 
	\end{align*}
	and Lemma \ref{Lem_Select_r1} complete the proof. 
\end{proof}

\begin{lem}
	\label{Lem_Existence_T0}
	Let $r_0 = r_0(s) = (\log)^4(s)$ and let $r_1 = r_1(s) = \tilde{r}_1(r_0(s), 1)$. 
	Then, for sufficiently large $s > 0$, $\varepsilon$ defined by (\ref{Condi_on_Epsilon}) satisfies 
	\begin{align}
		\label{Condi_epsilon}
		\frac{1}{(\log)^2(s)} 
		\leq \varepsilon 
		\leq \frac{e^{- 2 M}}{48 d^2}
		\leq 1.
	\end{align}
\end{lem}

\begin{proof}
	According to (\ref{Condi_on_r1}), for sufficiently large $s > 0$, we have $3m^{-1}(M+1) r_0(s) \leq r_1(s) \leq 5m^{-1}(M+1) r_0(s)$. 
	Therefore, $\varepsilon \geq |(\log)^2(s)|^{-1}$ holds by (\ref{Condi_on_Epsilon}). 
\end{proof}

\begin{lem}
	\label{Lem_Diff_Noisy_Non_Noisy}
	For any $\delta > 0$, let $\xi(\delta) = \inf \{ t \geq 0 \mid \| X_t^x - Y^x(s, t) \|_{\R^d} \geq \delta \}$. 
	Then, we have
	\begin{align*}
		P(\xi(\delta) < t)
		\leq \frac{4 e^{M t} d^2}{\delta} \sqrt{\frac{t}{\pi \gamma(s)}} \exp \left\{ - \frac{e^{- 2 M t} \delta^2 \gamma(s)}{4 t d^2} \right\},\quad t > 0. 
	\end{align*}
\end{lem}

\begin{proof}
	By the definitions of $X^x$ and $Y^x(s, \cdot)$, 
	\begin{align*}
		\| X_t^x - Y^x(s, t) \|_{\R^d}
		&\leq M \int_0^t \| X_u^x - Y^x(s, u) \|_{\R^d} du 
		+ \sqrt{\frac{2}{\gamma(s)}} \| W_t \|_{\R^d}
	\end{align*}
	holds, and therefore we obtain
	\begin{align*}
		\| X_t^x - Y^x(s, t) \|_{\R^d} 
		\leq \sqrt{\frac{2}{\gamma(s)}} e^{M t} \max_{0 \leq u \leq t} \| W_u \|_{\R^d}
	\end{align*}
	by Gronwall's lemma. 
	Thus, $P(\xi(\delta) < t) \leq P(\max_{0 \leq u \leq t} \| X_u^x - Y^x(s, u) \|_{\R^d} \geq \delta ) \leq P( \max_{0 \leq u \leq t} \| W_u \|_{\R^d} \geq e^{-M t} \delta \sqrt{\gamma(s) / 2} )$. 
	Applying Problem 2.8.3 in \citep{kara} to the R.H.S. of 
	\begin{align*}
		P\left( \max_{0 \leq u \leq t} \| W_u \|_{\R^d} \geq e^{-M t} \delta \sqrt{\gamma(s) / 2} \right)
		\leq \sum_{i=1}^d P\left( \max_{0 \leq u \leq t} | W_{i, u} | \geq e^{-M t} \delta \sqrt{\gamma(s) / 2 d^2} \right),  
	\end{align*}
	we obtain the desired result. 
\end{proof}

\begin{lem}
	\label{Lem_Markov_Time_Sequence}
	Let $r_0 = r_0(s)$ and $r_1 = r_1(s)$ be the same as in Lemma \ref{Lem_Existence_T0} and let $r_2 = r_2(s) = r_1(s) + 6$. 
	In addition, for any continuous process $V$, let 
	\begin{align}
		\label{Def_Exit_Omega2}
		\tau(V) = \inf \{ t \geq 0 \mid V_t \notin \Omega_{2, F} \}.
	\end{align}
	Then, for sufficiently large $s > 0$, the following inequality holds. 
	\begin{align}
		\label{IQ_Exit_Time_Prob_Bound}
		&P\left(\tau(Y^x(s, \cdot)) < (\log)^2(s) \right) 
		\leq \frac{2}{(\log)^2(s)} \left( 1 + \frac{8 e^M d^2}{\sqrt{\pi \gamma(s)}} \right),\quad x \in \Omega_{0, F}.
	\end{align}
\end{lem}

\begin{proof}
	In this proof, we denote the underlying filtration as $\{ \F_t \}$. 	
	First, we only have to show (\ref{IQ_Exit_Time_Prob_Bound}) for $x \in \partial \Omega_{1, F}$. 
	In fact, denoting $\theta(V) = \inf \{ t \geq 0 \mid V_t \in \partial \Omega_{1, F} \}$, $\theta(V) \leq \tau(V)$ holds for $x \in \Omega_{0, F}$. 
	Thus, if (\ref{IQ_Exit_Time_Prob_Bound}) is true for all $x \in \partial \Omega_{1, F}$, then 
	\begin{align*}
		P\left( \tau(Y^x(s, \cdot)) < (\log)^2(s) \right)
		&= E\left[ P\left( \tau(Y^x(s, \cdot)) < (\log)^2(s) \,\Big|\, \F_\theta \right); \{ \theta < (\log)^2(s) \} \right], 
	\end{align*}
	and therefore the strong Markov property of $Y^x(s, \cdot)$ yields the desired result. 
	
	To show (\ref{IQ_Exit_Time_Prob_Bound}) for $x \in \partial \Omega_{1, F}$, we define the sequence of stopping times as $\sigma_0(V) = 0$, $\theta_0(V) = 0$ and
	\begin{align*}
		\sigma_{i+1}(V) 
		= \inf \{ t > \theta_i(V) \mid V_t \in \Omega_{0, F} \},\quad
		\theta_i(V) 
		= \inf \{ t > \sigma_i(V) \mid V_t \notin \Omega_{1, F} \},\qquad i \geq 1.
	\end{align*}
	Let  
	\begin{align*}
		Q_F(t, V)
		&
		= \exp \left\{ \frac{\gamma(s)}{2} F(V_t) - \frac{\gamma(s)}{2} F(V_0) - \frac{1}{2} \int_0^t \Delta F(V_s) ds + \frac{\gamma(s)}{4} \int_0^t \| \nabla F(V_s) \|_{\R^d}^2 ds \right\}.
	\end{align*}
	Then, by Ito's formula, we have
	\begin{align*}
		Q_F(t, Y^x(s, \cdot)) 
		&= \exp \left\{ \sqrt{\frac{\gamma(s)}{2}} \int_0^t \langle \nabla F(Y^x(s, u)), dW_u \rangle_{\R^d} - \frac{\gamma(s)}{4} \int_0^t \| \nabla F(Y^x(s, u)) \|_{\R^d}^2 du \right\}. 
	\end{align*}
	Therefore, By Girsanov's theorem, $Y^x(s, \cdot)$ on $[0, \tau(Y^x(s, \cdot))]$ under the $Q_F(\tau(Y^x(s, \cdot)), Y^x(s, \cdot)) dP$ has the same distribution as $x + \sqrt{2 / \gamma(s)} W$. 
	
	On the other hand, by $x \in \partial \Omega_{1, F}$, if a continuous process $V$ satisfies $V_0 = x$, then $F(V_u) \geq r_0(s)$ holds for any $u \leq \sigma_1(V)$. 
	Thus, for sufficiently large $s > 0$, Lemma \ref{Lem_Select_r0} yields
	\begin{align*}
		\| \nabla F(V_u) \|_{\R^d}^2 - \frac{2}{\gamma(s)} \Delta F(V_u) 
		\geq 0.
	\end{align*}
	Therefore, by $\gamma(s) = (\log)^3(s)$ and $r_2(s) - r_1(s) = 6$, the following inequality holds on $\{ \tau(V) < \sigma_1(V) \}$. 
	\begin{align*}
		Q_F(\tau(V), V)^{-1}
		&= \exp \left\{  \frac{\gamma(s)}{2} F(x) - \frac{\gamma(s)}{2} F(V_{\tau(V)}) + \frac{1}{2} \int_0^{\tau(V)} \Delta F(V_u) du - \frac{\gamma(s)}{4} \int_0^{\tau(V)} \| \nabla F(V_u) \|_{\R^d}^2 du \right\} \\
		&\leq \frac{1}{|(\log)^2(s)|^3}. 
	\end{align*}
	In particular, since $\{ \tau(Y^x(s, \cdot)) < \sigma_1(Y^x(s, \cdot)) \} \in \F_{\tau(Y^x(s, \cdot))}$, denoting $\tilde{W} = x + \sqrt{2 / \gamma(s)} W$, we have 
	\begin{align*}
		P(\tau(Y^x(s, \cdot)) < \sigma_1(Y^x(s, \cdot))) 
		&= E\left[ Q_{F, s}(\tau(\tilde{W}), \tilde{W})^{-1}; \{ \tau(\tilde{W}) < \sigma_1(\tilde{W}) \} \right] 
		\leq \frac{1}{|(\log)^2(s)|^3}.
	\end{align*}
	Combining this inequality and the strong Markov property of $Y^x(s, \cdot)$, for all $k \in \N$, we obtain 
	\begin{align}
		P(\tau(Y^x(s, \cdot)) < \sigma_k(Y^x(s, \cdot))) 
		&= \sum_{i=1}^k P( \sigma_{i-1}(Y^x(s, \cdot)) \leq \tau(Y^x(s, \cdot)) < \sigma_i(Y^x(s, \cdot)) ) \notag \\
		&= \sum_{i=1}^k E[ P(\tau(Y^x(s, \cdot)) < \sigma_i(Y^x(s, \cdot)) \,|\, \F_{\theta_{i-1}(Y^x(s, \cdot))}); \{ \sigma_{i-1}(Y^x(s, \cdot)) \leq \tau(Y^x(s, \cdot)) \} ] \notag \\
		&\leq \frac{k}{|(\log)^2(s)|^3}. \label{IQ_tau_sigma}
	\end{align}
	
	If we define $\varepsilon$ and $\delta_0$ as (\ref{Condi_on_Epsilon}) and (\ref{Def_of_Delta0}), respectively, then $\xi(\delta_0)$ defined in Lemma \ref{Lem_Diff_Noisy_Non_Noisy} satisfies $P(\sigma_1(Y^x(s, \cdot)) < \varepsilon) \leq P(\xi(\delta_0) < \varepsilon)$. 
	In fact, since $Y^x(s, t_0) \in \Omega_{0, F}$ for $t_0 = \sigma_1(Y^x(s, \cdot))$, if $t_0 < \varepsilon$, then by the definition of $\delta_0$
	\begin{align*}
		\| Y^x(s, t_0) - X_{t_0}^x \|_{\R^d}
		\geq {\rm dist}\left( \Omega_{0, F}, {\textstyle \bigcup_{y \in \partial \Omega_{1, F}}} \Gamma_F^y(t_0) \right)
		\geq \delta_0
	\end{align*}
	holds. 
	Therefore, $\xi(\delta_0) \leq t_0 < \varepsilon$ by the definition of $\xi(\delta_0)$. 
	On the other hand, for sufficiently large $s > 0$, $\varepsilon$ satisfies (\ref{Condi_epsilon}). 
	Thus, Lemmas \ref{Lem_Lower_Bound_delta0} and \ref{Lem_Diff_Noisy_Non_Noisy} yield
	\begin{align*}
		P(\sigma_1(Y^x(s, \cdot)) < \varepsilon) 
		\leq \frac{8 e^M d^2}{\sqrt{\pi \gamma(s)}} \exp \left\{ - \frac{e^{- 2 M} \gamma(s)}{16 \varepsilon d^2} \right\}
		\leq \frac{8 e^M d^2}{|(\log)^2(s)|^3\sqrt{\pi \gamma(s)}}.
	\end{align*}
	
	Whereas, for any $k \in \N$, we have
	\begin{align*}
		\sigma_k(Y^x(s, \cdot)) 
		= \sigma_1(Y^x(s, \cdot)) + \sum_{i=1}^{k-1} (\sigma_{i+1}(Y^x(s, \cdot)) - \sigma_i(Y^x(s, \cdot)))
		\geq \sigma_1(Y^x(s, \cdot)) + \sum_{i=1}^{k-1} (\sigma_{i+1}(Y^x(s, \cdot)) - \theta_i(Y^x(s, \cdot))).
	\end{align*}
	Thus, on the event $\{ \sigma_k(Y^x(s, \cdot)) < k \varepsilon \}$, there exists at least one $0 \leq i \leq k-1$ such that $\sigma_{i+1}(Y^x(s, \cdot)) - \theta_i(Y^x(s, \cdot)) < \varepsilon$. 
	Therefore, since $P(\sigma_{i+1}(Y^x(s, \cdot)) - \theta_i(Y^x(s, \cdot)) < \varepsilon) \leq \sup_{y \in \partial \Omega_{1, F}} P(\sigma_1(Y^x(s, \cdot)) < \varepsilon)$ by the strong Markov property of $Y^x(s, \cdot)$, 
	\begin{align}
		\label{IQ_tau_k_epsilon}
		P(\sigma_k(Y^x(s, \cdot)) < k \varepsilon) 
		\leq \frac{8 k e^M d^2}{|(\log)^2(s)|^3\sqrt{\pi \gamma(s)}}
	\end{align}
	holds. 
	Combining (\ref{IQ_tau_sigma}) and (\ref{IQ_tau_k_epsilon}), we obtain  
	\begin{align*}
		P(\tau(Y^x(s, \cdot)) < k \varepsilon) 
		\leq P(\tau(Y^x(s, \cdot)) < \sigma_k(Y^x(s, \cdot))) + P(\sigma_k(Y^x(s, \cdot)) < k \varepsilon) 
		\leq \frac{k}{|(\log)^2(s)|^3} \left( 1 + \frac{8 e^M d^2}{\sqrt{\pi \gamma(s)}} \right). 
	\end{align*}
	As a result, taking $k \in \N$ so that $|(\log)^2(s)|^2 \leq k < |(\log)^2(s)|^2 + 1$, since $(\log)^2(s) \leq k \varepsilon$ holds by (\ref{Condi_epsilon}), we obtain
	\begin{align*}
		P\left( \tau(Y^x(s, \cdot)) < (\log)^2(s) \right) 
		\leq P\left( \tau(Y^x(s, \cdot)) < k \varepsilon \right) 
		\leq \frac{|(\log)^2(s)|^2 + 1}{|(\log)^2(s)|^3}  \left( 1 + \frac{8 e^M d^2}{\sqrt{\pi \gamma(s)}} \right)
		\leq \frac{2}{(\log)^2(s)} \left( 1 + \frac{8 e^M d^2}{\sqrt{\pi \gamma(s)}} \right), 
	\end{align*}
	as desired. 
\end{proof}

\begin{lem}
	\label{Lem_beta_s_t}
	The function $\alpha(s, \cdot)$ defined by (\ref{Def_alpha_s_t}) satisfies $\alpha(s, t) \geq s + t$. 
	In addition, if $s > 0$ is sufficiently large, then $\alpha(s, t) \leq s + 2 t$ holds for any $t \leq s$. 
\end{lem}

\begin{proof}
	For each fixed $s \geq 0$, the map $r \mapsto \int_s^r \frac{\gamma(s)}{\gamma(u)} du$ tends to infinity as $r \to \infty$. 
	Thus, $\alpha(s, t)$ is well-defined as the inverse of strictly increasing continuous function. 
	By the monotonicity of $\gamma(t)$, 
	\begin{align*}
		t
		= \int_s^{\alpha(s, t)} \frac{\gamma(s)}{\gamma(u)} du
		\leq \alpha(s, t) - s,
	\end{align*}
	holds, and therefore $s + t \leq \alpha(s, t)$. 
	
	For sufficiently large $s > 0$, we have $\gamma(s) = (\log)^3(s)$ and $2(\log)^3(s) \geq (\log)^3(3s)$. 
	Thus, for $t \leq s$, 
	\begin{align*}
		\int_s^{s + 2 t} \frac{(\log)^3(s)}{(\log)^3(u)} du
		\geq \frac{2 t (\log)^3(s)}{(\log)^3(s + 2 t)}
		\geq t
	\end{align*}
	holds. 
	Therefore, $\alpha(s, t) \leq s + 2 t$ follows from the definition of $\alpha(s, t)$.
\end{proof}

\begin{lem}
	\label{Lem_Diff_Between_Y_Z}
	Let $H \in C^1(\R^d; \R)$ be $M$-smooth. 
	If $r_0(s) = (\log)^4(s)$ and $h(s) \leq s^{2/3}$, then for any $x \in \Omega_{0, F}$, 
	\begin{align*}
		\left| E_{s, x}[H(Z_{\alpha(s, h(s))})] - E\left[ H(Y^x(s, h(s))) \right] \right|
		\leq O_{m, b, M, \gamma(0), H(0), \| \nabla H(0) \|_{\R^d}, F(0), \| \nabla F(0) \|_{\R^d}, d}\left( \frac{1 + \| x \|_{\R^d}^2}{\sqrt{ (\log)^2(s)}} \right)
	\end{align*}
	holds. 
	Here, $E_{s, x}[\cdot] = E[\cdot \,|\, Z_s = x]$. 
\end{lem}

\begin{proof}
	According to L\'{e}vy's theorem, 
	\begin{align*}
		\tilde{W}_t 
		\coloneqq \sqrt{\frac{\gamma(s)}{2}} \int_s^{\alpha(s, t)} \sqrt{\frac{2}{\gamma(u)}} dW_u
	\end{align*}
	is a new Brownian motion with respect to the time changed filtration. 
	Setting $u = \alpha(s, v)$, we have
	\begin{align*}
		\int_s^{\alpha(s, t)} \nabla F(Z_s) du 
		= \int_0^t \frac{\gamma(\alpha(s, u))}{\gamma(s)} \nabla F(\tilde{Z}(s, u)) du,\quad
		\int_s^{\alpha(s, t)} \sqrt{\frac{2}{\gamma(u)}} dW_u
		= \sqrt{\frac{2}{\gamma(s)}} \tilde{W}_t.
	\end{align*}
	Thus, when $Z_s = x$, $\tilde{Z}(s, t) = Z_{\alpha(s, t)}$ satisfies
	\begin{align}
		\tilde{Z}(s, t)
		&= x
		- \int_0^t \frac{\gamma(\alpha(s, u))}{\gamma(s)} \nabla F(\tilde{Z}(s, u)) du 
		+ \sqrt{\frac{2}{\gamma(s)}} \tilde{W}_t. \label{EQ_Transformed_Z} 
	\end{align}
	To apply the result of Section 7.6.4 in \citep{Liptser} to (\ref{EQ_Transformed_Z}) and
	\begin{align*}
		Y^x(s, t)
		&= x 
		- \int_0^t \nabla F(Y^x((s, u)) du 
		+ \sqrt{\frac{2}{\gamma(s)}} W_t,  
	\end{align*}
	let 
	\begin{align*}
		S_1(t) 
		&= - \sqrt{\frac{\gamma(s)}{2}} \int_0^t \left( \frac{\gamma(\alpha(s, u))}{\gamma(s)} -1 \right) \langle \nabla F(Y^x(s, u)), dW_u \rangle_{\R^d}, \\
		S_2(t) 
		&= \frac{\gamma(s)}{2} \int_0^t \left( \frac{\gamma(\alpha(s, u))}{\gamma(s)} -1 \right)^2 \| \nabla F(Y^x(s, u)) \|_{\R^d}^2 du, 
	\end{align*}
	and let $Q(t) = \exp \left\{ S_1(t) - \frac{1}{2} S_2(t) \right\}$. 
	Then for $\tau(Y^x(s, \cdot))$ defined by (\ref{Def_Exit_Omega2}), 
	\begin{align*}
		E_{s, x}[H(\tilde{Z}(s, h(s)))] - E[H(Y^x(s, h(s)))] 
		&= \left\{ E_{s, x}[H(\tilde{Z}(s, h(s))); \{ \tau(\tilde{Z}) \geq h(s) \}] - E[H(Y^x(s, h(s))); \{ \tau(Y^x(s, \cdot)) \geq h(s) \}] \right\} \\
		&\quad+ \left\{ E[H(\tilde{Z}(s, h(s))); \{ \tau(\tilde{Z}) < h(s) \}] - E[H(Y^x(s, h(s))); \{ \tau(Y^x(s, \cdot)) < h(s) \}] \right\} \\
		&\eqqcolon I_1 + I_2
	\end{align*}
	holds. 
	In the rest of proof, we bound each of $I_1$ and $I_2$. 
	
	First, we bound $I_1$. 
	Since $Q(t)$ is a martingale on $[0, \tau(Y^x(s, \cdot))]$, by Lemma \ref{Lp_bound_SGLD}, we have 
	\begin{align*}
		| I_1 |
		&\leq \sqrt{E[H(Y^x(s, h(s)))^2]} \sqrt{E[| Q(h(s) \wedge \tau(Y^x(s, \cdot))) - 1 |^2]} \\
		&\leq O_{m, b, M, \gamma(0), H(0), \| \nabla H(0) \|_{\R^d}, \| \nabla F(0) \|_{\R^d}, d}\left( (1 + \| x \|_{\R^d}) \sqrt{E[| Q(h(s) \wedge \tau(Y^x(s, \cdot))) - 1 |^2]} \right).
	\end{align*}
	When $h(s) \leq \tau(Y^x(s, \cdot))$, setting $v = \alpha(s, u)$, we obtain
	\begin{align*}
		S_2(h(s) \wedge \tau(Y^x(s, \cdot))) 
		&\leq \frac{\gamma(s)}{2} (\| \nabla F(0) \|_{\R^d} + M r_2(s))^2 \int_0^{h(s) \wedge \tau(Y^x(s, \cdot))} \left( \frac{\gamma(\alpha(s, u))}{\gamma(s)} - 1 \right)^2 du \\
		&= \frac{\gamma(s)}{2} (\| \nabla F(0) \|_{\R^d} + M r_2(s))^2 \int_s^{\alpha(s, h(s))} \left( \frac{\gamma(u)}{\gamma(s)} - 1 \right)^2 \frac{\gamma(s)}{\gamma(u)} du \\
		&\leq \frac{1}{2 \gamma(s)} (\| \nabla F(0) \|_{\R^d} + M r_2(s))^2 \int_s^{\alpha(s, t)} \left( \gamma(u) - \gamma(s) \right)^2 du. 
	\end{align*}
	Furthermore, since we have $0 \leq \log (1 + r) \leq r$ for all $r \geq 0$, 
	\begin{align*}
		| (\log)^k(u) - (\log)^k(s) |
		= \left| \log \left( 1 + \frac{(\log)^{k-1}(u)}{(\log)^{k-1}(s)} -1 \right) \right|
		\leq \frac{1}{(\log)^{k-1}(s)} | (\log)^{k-1}(u) - (\log)^{k-1}(s) | 
	\end{align*}
	holds for any $k$. 
	Therefore, since $\gamma(s) = (\log)^3 (s)$ and $r_1(s) \leq  O_{m, b, M, F(0), \| \nabla F(0) \|_{\R^d}}\left( (\log)^4 (s) \right)$ hold for sufficiently large $s > 0$, Lemma \ref{Lem_beta_s_t} yields
	\begin{align*}
		S_2(h(s) \wedge \tau(Y^x(s, \cdot))) 
		&\leq O_{m, b, M, F(0), \| \nabla F(0) \|_{\R^d}}\left( \frac{|(\log)^4 (s)|^2}{(\log)^3 (s)} \int_s^{\alpha(s, h(s))} \left| (\log)^3(u) - (\log)^3(s) \right|^2 du \right) \\
		&\leq O_{m, b, M, F(0), \| \nabla F(0) \|_{\R^d}}\left( \frac{|(\log)^4 (s)|^2}{s^2 |\log s|^2 |(\log)^2(s)|^2 (\log)^3 (s)} \int_s^{\alpha(s, h(s))} (u - s)^2 du \right) \\
		&\leq O_{m, b, M, F(0), \| \nabla F(0) \|_{\R^d}}\left( \frac{|(\log)^4 (s)|^2}{|\log s|^2 |(\log)^2(s)|^2 (\log)^3 (s)} \right).
	\end{align*}
	Whereas, by the martingale property of $\tilde{Q}(t) = \exp \left\{ 2 S_1(t) - 2 S_2(t) \right\}$ on $[0, \tau(Y^x(s, \cdot))]$, we obtain
	\begin{align*}
		E[| Q(h(s) \wedge \tau(Y^x(s, \cdot))) - 1 |^2] 
		&= E[\tilde{Q}(h(s) \wedge \tau(Y^x(s, \cdot))) \left( \exp \left\{ S_2(h(s) \wedge \tau(Y^x(s, \cdot))) \right\} - 1 \right)]. 
	\end{align*}
	
	As a result, since we have by $e^r - 1 \leq r e^r$
	\begin{align*}
		E[| Q(h(s) \wedge \tau(Y^x(s, \cdot))) - 1 |^2] 
		&= E[| Q(h(s) \wedge \tau(Y^x(s, \cdot))) - 1 |^2] \\
		&\leq O_{m, b, M, F(0), \| \nabla F(0) \|_{\R^d}}\left( \frac{|(\log)^4 (s)|^2}{|\log s|^2 |(\log)^2(s)|^2 (\log)^3 (s)} \right) 
	\end{align*}
	for sufficiently large $s > 0$, 
	\begin{align}
		\label{IQ_Bound_I1}
		|I_1| 
		\leq O_{m, b, M, \gamma(0), H(0), \| \nabla H(0) \|_{\R^d}, F(0), \| \nabla F(0) \|_{\R^d}, d}\left( \frac{1 + \| x \|_{\R^d}}{(\log)^2(s)} \right)
	\end{align}
	holds as desired. 
	
	Finally, we bound $I_2$. 
	Applying (\ref{IQ_Bound_I1}) to $H = 1$, we obtain
	\begin{align*}
		| P(\tau(\tilde{Z}) < h(s)) - P(\tau(Y^x(s, \cdot)) < h(s)) |
		&= | P(\tau(\tilde{Z}) \geq h(s)) - P(\tau(Y^x(s, \cdot)) \geq h(s)) | \\
		&\leq O_{m, b, M, \gamma(0), F(0), \| \nabla F(0) \|_{\R^d}, d}\left( \frac{1 + \| x \|_{\R^d}}{(\log)^2(s)} \right).
	\end{align*}
	Therefore, by Lemmas \ref{Lem_Markov_Time_Sequence} and \ref{Lp_bound_SGLD}, 
	\begin{align*}
		|I_2| 
		&\leq \sqrt{E[H(\tilde{Z}(s, h(s)))^2]} \sqrt{P(\tau(\tilde{Z}) < h(s))} + \sqrt{E[H(Y^x(s, h(s)))^2]} \sqrt{P(\tau(Y^x(s, \cdot)) < h(s))} \\
		&\leq O_{m, b, M, \gamma(0), H(0), \| \nabla H(0) \|_{\R^d}, F(0), \| \nabla F(0) \|_{\R^d}, d}\left( \frac{1 + \| x \|_{\R^d}^2}{\sqrt{ (\log)^2(s)}}  \right)
	\end{align*}
	holds, and therefore the proof is completed. 
\end{proof}

\subsection{Moment bound}

The following two lemmas can be proved in the similar manners to \citep{Cucumber} noting that $\gamma$ and $\eta$ are monotonic. 

\begin{lem}
	\label{Lp_bound_SGLD}
	(Lemma A.4 in \citep{Cucumber})
	Let $p \geq 2$ and let $F \in C^1(\R^d; \R)$ be $(m, b)$-dissipative and $M$-smooth. 
	Suppose that $Z$ is the solution of 
	\begin{align*}
		dZ_t
		= - \nabla F(X_t) dt + \sqrt{2 / \gamma(t)} dW_t
	\end{align*}
	with initial value $Z_0 \in L^p(\Omega; \R^d)$. 
	Then, for any $t \geq 0$, 
	\begin{align}
		\label{Lp_Bound_SGLD}
		E[\| Z_t \|_{\R^d}^p] 
		\leq e^{-\lambda(p) t} E[\| Z_0 \|_{\R^d}^p] + \frac{C(p)}{\lambda(p)} (1 - e^{-\lambda(p) t}) 
	\end{align}
	holds. 
	Here, $C(p)$ and $\lambda(p)$ are constants defined by (\ref{p-th_Lyap_Constant}). 
\end{lem}

\begin{lem}
	\label{Lp_bound_SGD}
	(Lemma A.5 in \citep{Cucumber})
	Assume that $F_k \in C^1(\R^d; \R)$ is $(m, b)$-dissipative and $M$-smooth for each $k$ and satisfies $\sup_{k \in \N} \| \nabla F_k(0) \|_{\R^d} \leq A$. 
	For a sequence $\eta = \{ \eta_k \}_{k=1}^\infty$ that decreases to $0$, let $Z^{(\eta)}$ be the process defined by
	\begin{align*}
		d Z_t^{(\eta)} 
		= - \nabla F_k(Z_{\phi^{(\eta)}(t)}^{(\eta)}) + \sqrt{2 / \gamma(t)} dW_t.
	\end{align*}
	Then, for all $\ell \in \N$, 
	\begin{align*}
		\sup_{t \geq 0} E[\| Z_t^{(\eta)} \|_{\R^d}^{2 \ell}]
		\leq O_{m, b, M, \gamma(0), A, d, \ell, \eta} (1 + E[\| Z_0 \|_{\R^d}^{2 \ell}])
	\end{align*}
	holds. 
\end{lem}

\subsection{Results on generalization bound}

\begin{thm}
	\label{Thm_Covering_Num_Bound}
	(Theorem 4.5 in \citep{Ren})
	Let $\F$ be a family of fuctions from $\mathcal{Z}$ to $\R$. 
	Denoting
	\begin{align*}
		\| f - \tilde{f} \|_{1, S} 
		\coloneqq \frac{1}{n} \sum_{i=1}^n |f(z_i) - \tilde{f}(z_i)|,\quad f, \tilde{f} \in \F, 
	\end{align*}
	let, $C(\F, \varepsilon, \| \cdot \|_{1, S})$ be the size of minimal $\varepsilon$-cover of $\F$ with respect to $\| \cdot \|_{1, S}$. 
	Then, if 
	\begin{align*}
		\sup_{f \in \F} \left( \frac{1}{n} \sum_{i=1}^n f(z_i)^2 \right)^{1/2}
		\leq c
	\end{align*}
	holds, we have
	\begin{align*}
		\hat{R}_n(\F, S)
		\leq \inf_{\varepsilon > 0} \left( \varepsilon + \frac{c \sqrt{2}}{\sqrt{n}} \sqrt{\log C(\F, \varepsilon, \| \cdot \|_{1, S})} \right), 
	\end{align*}
	where for IIDs $\sigma_1, \dots, \sigma_n$ satisfying $P(\sigma_i = 1) = P(\sigma_i = -1) = 1/2$, the empirical Rademacher complexity $\hat{R}_n(\F, S)$ is defined by 
	\begin{align}
		\label{Emp_Rademach_Complex}
		\hat{R}_n(\F, S) 
		= \frac{1}{n} E \left[ \sup_{f \in \F} \sum_{i=1}^n \sigma_i f(z_i) \,\Big|\, S \right].
	\end{align}
\end{thm}

\begin{thm}
	\label{Thm_Rademach_GenBound}
	(Theorem 4.1 in \citep{Ren})
	Let $R_n(\F) = E[\hat{R}_n(\F, S)]$. 
	Then we have
	\begin{align*}
		E \left[ \sup_{f \in \F} \left| E[f(z_1)] - \frac{1}{n} \sum_{i=1}^n f(z_i) \right|\right]
		\leq 4 R_n(\F).
	\end{align*}
\end{thm}

\bibliographystyle{plain}
\bibliography{article.bib}

\begin{thebibliography}{10}

\bibitem{Laarhoven}
Emile~HL Aarts et~al.
\newblock Simulated annealing: Theory and applications.
\newblock 1987.

\bibitem{Abbas}
Houssam Abbas and Georgios Fainekos.
\newblock Convergence proofs for simulated annealing falsification of safety
  properties.
\newblock In {\em 2012 50th Annual Allerton Conference on Communication,
  Control, and Computing (Allerton)}, pages 1594--1601. IEEE, 2012.

\bibitem{Bertsekas}
Dimitri~P Bertsekas and John~N Tsitsiklis.
\newblock Gradient convergence in gradient methods with errors.
\newblock {\em SIAM Journal on Optimization}, 10(3):627--642, 2000.

\bibitem{Bouttier}
Cl{\'e}ment Bouttier and Ioana Gavra.
\newblock Convergence rate of a simulated annealing algorithm with noisy
  observations.
\newblock {\em The Journal of Machine Learning Research}, 20(1):127--171, 2019.

\bibitem{Kumar}
Huy~N Chau, Chaman Kumar, Mikl{\'o}s R{\'a}sonyi, and Sotirios Sabanis.
\newblock On fixed gain recursive estimators with discontinuity in the
  parameters.
\newblock {\em ESAIM: Probability and Statistics}, 23:217--244, 2019.

\bibitem{Moulines}
Ngoc~Huy Chau, {\'E}ric Moulines, Miklos R{\'a}sonyi, Sotirios Sabanis, and
  Ying Zhang.
\newblock On stochastic gradient langevin dynamics with dependent data streams:
  The fully nonconvex case.
\newblock {\em SIAM Journal on Mathematics of Data Science}, 3(3):959--986,
  2021.

\bibitem{Abbasi}
Xiang Cheng, Niladri~S Chatterji, Yasin Abbasi-Yadkori, Peter~L Bartlett, and
  Michael~I Jordan.
\newblock Sharp convergence rates for langevin dynamics in the nonconvex
  setting.
\newblock {\em arXiv preprint arXiv:1805.01648}, 2018.

\bibitem{Mackey2}
Murat~A Erdogdu, Lester Mackey, and Ohad Shamir.
\newblock Global non-convex optimization with discretized diffusions.
\newblock {\em Advances in Neural Information Processing Systems}, 31, 2018.

\bibitem{Zhou}
Xuefeng Gao, Zuo~Quan Xu, and Xun~Yu Zhou.
\newblock State-dependent temperature control for langevin diffusions.
\newblock {\em arXiv preprint arXiv:2011.07456}, 2020.

\bibitem{Ge}
Rong Ge, Furong Huang, Chi Jin, and Yang Yuan.
\newblock Escaping from saddle points-online stochastic gradient for tensor
  decomposition.
\newblock In {\em Conference on learning theory}, pages 797--842. PMLR, 2015.

\bibitem{Mitter}
Saul~B Gelfand and Sanjoy~K Mitter.
\newblock Recursive stochastic algorithms for global optimization in
  $\mathbb{R}^d$.
\newblock {\em SIAM Journal on Control and Optimization}, 29(5):999--1018,
  1991.

\bibitem{Hadt}
Moritz Hardt, Ben Recht, and Yoram Singer.
\newblock Train faster, generalize better: Stability of stochastic gradient
  descent.
\newblock In {\em International conference on machine learning}, pages
  1225--1234. PMLR, 2016.

\bibitem{Sheu}
Chii-Ruey Hwang, Tzuu-Shuh Chiang, and Shuenn-Jyi Sheu.
\newblock Diffusion for global optimization in $\mathbb{R}^n$.
\newblock {\em Siam Journal on Control and Optimization}, 25:737--753, 1987.

\bibitem{Ge2}
Chi Jin, Rong Ge, Praneeth Netrapalli, Sham~M Kakade, and Michael~I Jordan.
\newblock How to escape saddle points efficiently.
\newblock In {\em International Conference on Machine Learning}, pages
  1724--1732. PMLR, 2017.

\bibitem{kara}
Ioannis Karatzas and Steven Shreve.
\newblock {\em Brownian motion and stochastic calculus}, volume 113.
\newblock Springer Science \& Business Media, 2012.

\bibitem{Lecchini}
Andrea Lecchini-Visintini, John Lygeros, and Jan Maciejowski.
\newblock Simulated annealing: Rigorous finite-time guarantees for optimization
  on continuous domains.
\newblock {\em Advances in Neural Information Processing Systems}, 20, 2007.

\bibitem{Qian}
Chris~Junchi Li, Lei Li, Junyang Qian, and Jian-Guo Liu.
\newblock Batch size matters: A diffusion approximation framework on nonconvex
  stochastic gradient descent.
\newblock {\em stat}, 1050:22, 2017.

\bibitem{Liptser}
Robert~Shevilevich Liptser and Al'bert~Nikolaevich Shiriaev.
\newblock {\em Statistics of random processes: General theory}, volume 394.
\newblock Springer, 1977.

\bibitem{Locatelli}
Marco Locatelli.
\newblock Simulated annealing algorithms for continuous global optimization:
  convergence conditions.
\newblock {\em Journal of Optimization Theory and applications},
  104(1):121--133, 2000.

\bibitem{Majka}
Mateusz~B Majka, Aleksandar Mijatovi{\'c}, and {\L}ukasz Szpruch.
\newblock Nonasymptotic bounds for sampling algorithms without log-concavity.
\newblock {\em The Annals of Applied Probability}, 30(4):1534--1581, 2020.

\bibitem{Kavis}
Panayotis Mertikopoulos, Nadav Hallak, Ali Kavis, and Volkan Cevher.
\newblock On the almost sure convergence of stochastic gradient descent in
  non-convex problems.
\newblock {\em Advances in Neural Information Processing Systems},
  33:1117--1128, 2020.

\bibitem{Mou}
Wenlong Mou, Liwei Wang, Xiyu Zhai, and Kai Zheng.
\newblock Generalization bounds of sgld for non-convex learning: Two
  theoretical viewpoints.
\newblock In {\em Conference on Learning Theory}, pages 605--638. PMLR, 2018.

\bibitem{Taiji}
Boris Muzellec, Kanji Sato, Mathurin Massias, and Taiji Suzuki.
\newblock Dimension-free convergence rates for gradient langevin dynamics in
  rkhs.
\newblock {\em arXiv preprint arXiv:2003.00306}, 2020.

\bibitem{Ragi}
Maxim Raginsky, Alexander Rakhlin, and Matus Telgarsky.
\newblock Non-convex learning via stochastic gradient langevin dynamics: a
  nonasymptotic analysis.
\newblock In {\em Conference on Learning Theory}, pages 1674--1703. PMLR, 2017.

\bibitem{Ren}
L.~Renjie.
\newblock Notes on rademacher complexity.
\newblock
  \url{http://www.cs.toronto.edu/~rjliao/notes/Notes_on_Rademacher_Complexity.pdf},,
  2022, May 25.

\bibitem{Cucumber}
Keisuke Suzuki.
\newblock Weak convergence of approximate reflection coupling and its
  application to non-convex optimization.
\newblock {\em arXiv preprint arXiv:2205.11970}, 2022.

\bibitem{Taiji2}
Taiji Suzuki.
\newblock Generalization bound of globally optimal non-convex neural network
  training: Transportation map estimation by infinite dimensional langevin
  dynamics.
\newblock {\em Advances in Neural Information Processing Systems},
  33:19224--19237, 2020.

\bibitem{Xu}
Pan Xu, Jinghui Chen, Difan Zou, and Quanquan Gu.
\newblock Global convergence of langevin dynamics based algorithms for
  nonconvex optimization.
\newblock {\em Advances in Neural Information Processing Systems}, 31, 2018.

\bibitem{Zhang}
Ying Zhang, {\"O}mer~Deniz Akyildiz, Theodoros Damoulas, and Sotirios Sabanis.
\newblock Nonasymptotic estimates for stochastic gradient langevin dynamics
  under local conditions in nonconvex optimization.
\newblock {\em arXiv preprint arXiv:1910.02008}, 2019.

\bibitem{Liang}
Yuchen Zhang, Percy Liang, and Moses Charikar.
\newblock A hitting time analysis of stochastic gradient langevin dynamics.
\newblock In {\em Conference on Learning Theory}, pages 1980--2022. PMLR, 2017.

\end{thebibliography}

\clearpage

\end{document}